\documentclass[a4paper,11pt,oneside]{article}
\usepackage[margin=1.2in]{geometry}

\usepackage{amsfonts}
\usepackage{dsfont}
\usepackage{amsmath,amssymb,amsthm}
\usepackage{algorithm,algpseudocode}

\usepackage{amsfonts,amssymb, amsmath, amsthm,
			latexsym,
			mathrsfs,
			url,
			verbatim}
\usepackage{graphicx,  subcaption}

\usepackage{algorithm, algorithmicx,algpseudocode}\usepackage{natbib}

\usepackage{multirow}

\newcommand{\bbR} 	{\mathbb{R}}
\newcommand{\calP}	{\mathcal{P}}
\newcommand{\cN}	{\mathcal{N}}
\newcommand{\cR}	{\mathcal{R}}
\newcommand{\C}		{\mathcal{C}}
\newcommand{\F}		{\mathcal{F}}
\renewcommand{\H}	{\mathcal{H}}
\newcommand{\R}   	{\mathbb{R}}
\newcommand{\X}		{\mathcal{X}}
\newcommand{\Y}		{\mathcal{Y}}
\newcommand{\Z}		{\mathcal{Z}}

\newcommand{\Ex}{\mathbb{E}}
\newcommand{\E}{\mathbb{E}}

\newcommand{\abs}[1]{\left| #1 \right|}

\renewcommand{\epsilon}	{\varepsilon}
\newcommand{\eps}{\varepsilon}

\newcommand{\lip}[1]{\nrm{#1}_{\textrm{{\tiny \textup{Lip}}}}}

\newcommand{\norm}[2][\cdot]{||#2||}
\newcommand{\nrm}[1]{\left\Vert #1 \right\Vert}

\newcommand{\beq}	{\begin{eqnarray*}}
\newcommand{\eeq}	{\end{eqnarray*}}
\newcommand{\beqn}	{\begin{eqnarray}}
\newcommand{\eeqn}	{\end{eqnarray}}
\newcommand{\ben}	{\begin{enumerate}}
\newcommand{\een}	{\end{enumerate}}
\newcommand{\bit}	{\begin{itemize}}
\newcommand{\eit}	{\end{itemize}}
\newcommand{\hide}[1]{}
\newcommand{\inv}{^{-1}} %

\newcommand{\OnePt}
{\textup{{\textsf{OnePointExtension}}}}
\newcommand{\LipSmooth}{\textup{{\textsf{LipschitzSmooth}}}}
\newcommand{\MultiPt}{\textup{{\textsf{MultiPointExtension}}}}

\newcommand{\set}[1]{\left\{ #1 \right\}}
\newcommand{\sset}[1]{\{ #1 \}}
\DeclareMathOperator*{\argmin}{argmin}
\DeclareMathOperator {\Ball}  {Ball}

\newcommand{\iprd}[1]{\langle #1 \rangle}

\newtheorem{theorem}{Theorem}[section]

\newtheorem{problem}[theorem]{Problem}

\usepackage{todonotes}

\title{Efficient Kirszbraun Extension with Applications to Regression}

\author{
Hanan Zaichyk\\Ben-Gurion University\\\texttt{zaichyk@post.bgu.ac.il}
\and
Armin Biess\\Ben-Gurion University\\\texttt{armin.biess@gmail.com}
\and  
Aryeh Kontorovich\\Ben-Gurion University\\\texttt{karyeh@bgu.ac.il}
\and
Yury Makarychev\\Toyota Technological Institute at Chicago\\\texttt{yury@ttic.edu}
}

\begin{document}

\maketitle

\begin{abstract}

We introduce a framework for performing regression between two Hilbert spaces. This is done based on Kirszbraun's extension theorem,
to the best of our knowledge, the first application of this technique to supervised learning. We analyze the statistical and computational aspects of this method. We decompose this task into two stages: training (which corresponds operationally to smoothing/regularization) and prediction (which is achieved via Kirszbraun extension). Both are solved algorithmically via a novel multiplicative weight updates (MWU) scheme, which, for our problem formulation, achieves a quadratic runtime improvement over the state of the art. Our empirical results indicate a dramatic improvement over standard off-the-shelf  solvers in our setting.

\end{abstract}

\section{Introduction.}\label{intro} %%1.

\paragraph{Regression.} 
The classical problem of estimating a continuous-valued function from noisy observations, known as regression, is of central importance in statistical theory with a broad range of applications; see, for example,  \citet{gyorfi2006distribution, nadaraya1989nonparametric}.  When the target function is assumed to have a specific structure, the regression problem is termed {\em parametric}
and the optimization problem is finite-dimensional.
%, means the main objective is to "tune" a set of %specific parameters according to the data. 
Linear regression \citep[chapter 10.3.1]{mohri2012foundations}
is perhaps the simplest and most common type of parametric regression. When no structural assumptions concerning the target function are made, the regression problem is 
%termed 
{\em nonparametric}. 
Informally, the main objective in the study of nonparametric regression is to understand the relationship between the regularity conditions that a function class might satisfy (e.g., Lipschitz or H\"older continuity, or sparsity in some representation)
and its behavior vis-\`{a}-vis optimization and generalization.
Most existing algorithms for regression
either focus on the scalar-valued case or else reduce multiple outputs to several scalar problems \citep{borchani2015survey}, 
see Related Work.
%\HZ{which make the statistical guarantees less accurate - formalize this sentence}.
%\AK{I would omit this bit, we discuss some of that in Rel. work}
%\HZ{I moved this part to 'our contribtuion': In this paper, we propose a method that directly exploits the metrics of the input and output spaces, which makes it explicitly sensitive to the interaction among the output components.}
%*\hz{We might need to explain here about ML, training/test error and ERM}.

\paragraph{Convex optimization.} 
Many learning problems can be cast in the framework of convex optimization. In particular, regression naturally lends itself to this formulation. While some cases, such as linear regression, admit efficient closed form solutions, this is not the case in general. 
Typically, convex optimization
problems are solved via
iterative methods
up to a specified accuracy. 
One general approach is the
interior-point methods,
which, on problems with $n$ variables
and $m$ constraints
achieves a runtime of
%for instance, is considered to be efficient %in practice, and in some cases are proved to perform in   
$O(\max\{n^3,n^2 m, F \})$, 
where $F$ id the cost of evaluating the
first and second derivatives of the objective and the constrains \citep{boyd2004covex}.

\paragraph{Motivation and contribution.}
%\HZ{
The chief motivation of this work was to generalize the results of \citet{gottlieb2017efficient},
who provided efficient nonparametric regression
methods in the scalar-output case.
Attempts to numerically solve our optimization problem,
which is naturally
formulated as a Quadratically Constrained Quadratic Program  (QCQP),
via state-of-the-art 
off the shelf solvers indicated that
these are incapable of handling
our framework,
even for relatively small data sets and dimensions.
This limitation of QCQP solvers
%\HZ{This gap 
motivated us to develop a specialized algorithm
to solve the optimization problem entailed
by our regression setting.
%to better support our optimization problem.}

In Section~\ref{sec:experiments}, we show that
our specialized algorithm dramatically
outperforms general-purposed QCQP solvers;
this algorithm, its theoretical analysis,
and MATLAB code\footnote{
available at
\url{https://github.com/HananZaichyk/Kirszbraun-extension}
}
%, thus motivated the need for a more efficient optimization algorithms for such problems.
are the main contributions of this paper.
%}
%\paragraph{Our contribution.}
%\HZ{
%added this part:
We  introduce  a  framework  for  performing  regression  between  two  Hilbert  spaces.  This  is  done  based  on Kirszbraun's extension theorem --- to the best of our knowledge, the first application of this technique to supervised learning. This method directly exploits the metrics of the input and output spaces, which makes it explicitly sensitive to the interaction among the output components. 
Although our main contributions are algorithmic,
a statistical analysis of our regression
technique is provided in Appendix
%More so, it enable us to statistically analyse the learner's output as we do in 
\ref{sec:gen-bds}.
%}

We formulate the regression problem
in two stages: {\em smoothing} and 
{\em extension},
which are formally described in
%Our smoothing and extension problems are formally stated in 
Section~\ref{sec:alg}. Roughly speaking,
on a dataset of size $n$ with $a$ input and $b$ output dimensions, we 
formulate the smoothing problem
as
a Quadratically Constrained Quadratic Program (QCQP) problem
with $bn$ variables and $O(n^2)$
constraints.
The extension problem is also
formulated as a QCQP with
$O(n)$ variables
and $O((a+b)n)$ constraints.

Although general QCQP problems are not
convex, our special instance
is, and as such is, in principle,
amenable to the standard
%does lend itself to the 
convex
optimization framework, such as
interior point methods.
When solving large-scale problems,
even a modest improvement in the exponent
yields dramatic runtime savings.
We propose a  
Multiplicative Weight Update (MWU) scheme to
solve the smoothing problem,
%up 
to a 
%precision $\eps$,
constant precision,
%to the Lipshchitz constraint,
in
runtime  $O(ma + {m^{3/2} (\log m)^2 b })$
and the extension problem in runtime
$O(na + nb \log n)$.

\paragraph{Related work.}
Previous approaches to vector-valued regression include
$\epsilon$-insensitive SVM with $p$-norm regularization
\citep{brudnak2006vector}, least-squares and MLE-based
methods \citep{jain2015alternating},
and (for linear models) the Danzig selector
\citep{chen2018improved}. According to a recent survey
\citep{borchani2015survey},
existing methods essentially ``transform the multi-output problem into independent single-output problems.''
%--- i.e., solve multiple tasks at the same time. 
Some approaches to multitask learning problems \citep{caruana1997multitask} exploit relations between the different tasks. In econometrics, this decoupling of the outputs is made explicit in the Seemingly Unrelated Regressions (SUR) model \citep{davidson1993estimation,greene2003econometric,greene2012econometric}.
 These approaches however, do not seem to encapsulate the need of a single vector output with possibly strong relations between its coordinates. In our approach, 
we devise a principled approach for leveraging the dependencies via Kirszbraun extension.
The latter has previously been applied
by \citet{mahabadi2018nonlinear}
to dimensionality reduction (unsupervised learning),
but to the best of our knowledge has not been used in the supervised learning setting.

Both of our problems (smoothing and extension)
may be formulated as 
QCQP programs,
whose most general form is
\begin{align*}
\mathrm{Minimize}\ & x^\top P_0 x+a^\top x\\
\mathrm{subject}\ to\ \ & x^\top P_ix + a_ix  \leq b_i ,i=1,\dots,m
\end{align*}
where $a$,$a_{i\in[m]}$ and $x$ are vectors, $P_0,P_{i\in[m]}$ are matrices, and the $b_i$ are scalars.
The general problem is NP-hard, but
when all of the $P_i$ are semi-definite, the problem is convex and can be solved in polynomial time
\citep{boyd2004covex}.
%\yury{(added the following)}
The QCQP is usually solved in practice using log-barrier or primal-dual interior-point methods. The running time of an optimization algorithm based on the interior-point methods significantly depends on the problem at hand. Specifically, consider a problem with $N$ variables and $m$ constraints. In order to obtain a $(1+\varepsilon)$-approximate solution, the algorithm has to perform $\Theta(\sqrt{m}\log (1/\varepsilon))$ iterations in the worst-case \citep[Chapter 6]{nesterov1994interior}. In each iteration, the algorithm has to initialize and invert an $N\times N$ Hessian matrix (or equivalently solve a system of $N$ linear equations with $N$ variables). 
%Now, t
The time required to initialize the Hessian matrix is problem specific: while it is $O(m N^2)$ in the worst case, it is often significantly less than that. The Hessian matrix can be inverted in $O(N^\omega)$ time, where $\omega$ is the matrix multiplication exponent \citep{BunchHopcroft74} (the best current upper bound on $\omega$ is $2.37286$ \citep{AlmanWilliams21}). However, to the best of our knowledge, all implementations used in practice perform this step in $O(N^3)$ time.
That said, this step can be significantly sped up if the Hessian matrix has a special structure.

Our Multiplicative Weight Update (MWU)
scheme is based on
the 
framework of \citet{Arora2012the}.
We include the relevant background and results in the Appendix for completeness.
%that we invoke in our methods.
\iffalse
We design algorithms for Kirszbraun extension problems \yury{We should define them above?} using
the Multiplicative Weight Update (MWU) framework \citet{Arora2012the}.
%Our proposed approximate solution for the QCQP problems takes leverage from \citet{kirszbraun34} guarantees, to implement an \textit{oracle} that solve a smaller problem, and allows us to use the Multiplicative Weight Update (MWU) framework described as in \citet{Arora2012the}.
\fi

\paragraph{Main results.}
We cast the general regression problem
between Hilbert spaces
as 
%a
two QCQP programs, 
%for which we 
and provide 
an
efficient algorithm for each problem.

The problem setup, formalized in
Section~\ref{sec:definitions},
involves a dataset of size $n$
of vectors in an $a$-dimensional Euclidean
space labeled by $b$-dimensional vectors.
The {\em smoothing} (also: training, regularization, denoising)
problem (Section~\ref{sec:ERM})
is to perturb the labels so
as to achieve the user-specified Lipschitz
smoothness constraint while incurring
a minimum distortion.
This is a standard statistical technique,
known as {\em regularization},
which prevents
%This is done in order to prevent
overfitting in prediction.
Our Theorem~\ref{thm:LipSmooth}
solves the smoothing optimization problem, up to a tolerance $\eps$, in runtime
$O(an^2+bn^3(\log n)^2\log(1/\eps)/\eps^{5/2})$.

Next,
we address
the task of prediction
(i.e., assigning a label to a test point).
In
Theorem~\ref{thm:lip_ext},
we accomplish this via
$\eps$-approximate 
Kirszbraun extension
of the smoothed dataset,
in runtime
$O(an+bn(\log n)/\eps^2)$.
For small $a$, an improvement is possible:
a data structure
can be constructed off-line
at a (once) runtime cost of $2^{O(a)}n\log n$
that allows to answer (multiple)
future prediction queries
in time 
$$(1/\eps)^{O(a)}b\log n\log\log n.$$

In Section~\ref{sec:experiments},
we compare the performance of our
MWU-based approach to a state of the art
interior-point based solver
and report a significant runtime advantage,
which allows to process larger samples and ultimately yields greater accuracy.

Finally, for completeness, 
in Section~\ref{sec:gen-bds},
we
include a Rademacher-based analysis of
the generalization error
of our regression algorithm.

\section{Formal setup.}\label{sec:definitions} %% 2.

\paragraph{Metric space.} A metric space $(X,d_X)$ is a set $X$ equipped with
a symmetric function $d_X:X^2\to[0,\infty)$
  satisfying
  $d_X(x,x')=0\iff x=x'$
  and
  the triangle inequality.
  Given two metric spaces $(X,d_X)$ and $(Y,d_Y)$,
  a function
$f:X\to Y$
  is {\em $L$-Lipschitz} if
  $d_Y(f(x),f(x'))\le Ld_X(x,x')$ for all $x,x'\in X$;
  its {\em Lipschitz} constant $\lip{f}$ is the smallest $L$
  for which the latter inequality holds.
  For any metric space $(X,d_X)$
  and $A\subseteq X$,
  the following
  classic {\em Lipschitz extension} result,
  essentially due to \citet{mcshane1934,Whitney1934},
  holds. If $f:A\to\R$ is Lipschitz (under the inherited metric)
  then there is an {\em extension} $f^*:X\to\R$ that coincides with
  $f$ on $A$ and $\lip{f}=\lip{f^*}$.
  A {\em Hilbert space} $H$ is a vector space (in our case, over $\R$)
  equipped with an inner product $\iprd{\cdot,\cdot}:H^2\to\R$,
  which is a
positive-definite symmetric bilinear form;
further, $H$ is complete in the metric $d_H(x,x'):=\sqrt{\iprd{x-x',x-x'}}$.
\paragraph{Kirszbraun theorem.}
\citet{Kirszbraun34} proved that for two Hilbert spaces $(X,d_X)$ and $(Y,d_Y)$, and $f$ mapping $A\subseteq X$ to $Y$,
there is an extension $f^*:X\to Y$ such that $\lip{f}=\lip{f^*}$. This result is in general false for Banach spaces whose norm is not induced by an inner product  \citep{naor15}.

\paragraph{Learning problem.}
We assume a familiarity with the abstract agnostic learning framework
and refer the reader to \citet{mohri2012foundations} for background.
%\armin{Our approach to identifying (static) corresponding postures between dissimilar robots is to learn a mapping between two Hilbert spaces, $(X,d_X)$ and $(Y,d_Y)$. WE DECIDED NOT TO ADDRESS THIS PROBLEM - NEEDS MODIFICATION.}
%Thus, 
Our approach will be applied to learn a mapping between two Hilbert spaces, $(X,d_X)$ and $(Y,d_Y)$. We assume a fixed unknown distribution $P$ on $X\times Y$
and a labeled sample $(x_i,y_i)_{i\in[n]}$ of input-output examples.
The {\em risk} of a given mapping $f:X\to Y$ is defined as
$R(f)=\Ex_{(x,y)\sim P}[d_Y(f(x),y)]$;
implicit here is our designation of the metric of $Y$ as the loss function.
Analogously, the empirical risk of $f
$ on a labeled sample
is given by $\hat R_n(f)=n\inv\sum_{i\in[n]}d_Y(f(x_i),y_i)$.
In this paper, we always take $X=\R^a$ and $Y=\R^b$, each equipped with
the standard Euclidean metric. Uniform deviation bounds on $|R(f)-\hat R_n(f)|$,
over all $f$ with $\lip{f}\le L$
are given in Section~\ref{sec:gen-bds}.

\section{Learning algorithm}
\label{sec:alg}

\paragraph{Overview.}
We follow the basic strategy proposed by \citet{gottlieb2017efficient} for real-valued regression.
We are given a labeled sample $(x_i,y_i)_{i\in[n]}$, where $x_i\in X:=\R^a$ and $y_i\in Y:=\R^b$.
For 
a user-specified
%each candidate 
Lipschitz constant $L>0$, we compute
the (approximate) Empirical Risk Minimizer (ERM) 
$\hat f:=\argmin_{f\in F_L}\hat R_n(f)$
over $F_L:=\sset{ f\in Y^X : \lip{f}\le L}$.
(A standard method for tuning $L$
is via
%Following the standard 
Structural Risk Minimization (SRM):
One computes a
generalization bound
$R(\hat f)\le \hat R_n(\hat f)+Q_n(a,b,L)$,
where $Q_n(a,b,L):=\sup_{f\in F_L}|R(f)-\hat R_n(f)|=O(L/n^{a+b+1})$,
as
derived in
in Section~\ref{sec:gen-bds},
and chooses
$\hat L$ to minimize this.
We omit this standard stage of the
learning process.)

Predicting the value at a test point $x^*\in X$ amounts to Lipschitz-extending
$\hat f$ from $\set{ x_i : i\in [n]}$ to $\set{ x_i : i\in [n]}\cup\set{x^*}$.
Equivalently, the ERM stage may be viewed as a {\em smoothing} procedure,
where
$\tilde y_i:=\hat f(x_i)$
and
$(x_i,\tilde y_i)_{i\in[n]}$
is the {\em smoothed sample} --- which is then (approximately) Lipschitz-extended to $x^*$.
We proceed to describe each stage in detail.

\subsection
    {Approximate Lipschitz extension}
    \label{sec:lipext}

\paragraph{Problem statement.}
Given a finite sequence $(x_i)_{i\in[n]}\subset X=\R^a$,
its image
$(y_i)_{i\in[n]}\subset Y=\R^b$
under some $L$-Lipschitz map $f:X\to Y$,
a test point $x^*$,
and a precision parameter $\eps>0$,
we wish to compute $y^*=f(x^*)$
so that
$\nrm{y^*-f(x_i)}\le(1+\eps)L\nrm{x^*-x_i}$
for all $i\in[n]$.
Our first result is an efficient algorithm for achieving this:

\begin{theorem}\label{thm:lip_ext}
  The approximate Lipschitz extension algorithm \OnePt\ has runtime
  $O(na + nb \log n /\eps^2)$.
\end{theorem}

The query runtime can be significantly improved if the dimension of $X$ is moderate:

\begin{theorem}\label{thm:low-dim-lip_ext}

  There is a data structure for the Lipschitz extension problem of
memory
  size $O(2^{O(a)}n)$ that can be constructed in time $O(2^{O(a)} n \log n)$.
Given a query point $x^*$ and a parameter $\eps \in (0,1/2)$, one can compute $y^*$ such that
$\|y^* - y_i\| \leq (1+\eps) L \|x^* - x_i\|$ for every $i$
in time $(1/\eps)^{O(a)} b\log n \log \log n$.
\end{theorem}

\paragraph{Analysis.}
We analyze algorithm \OnePt~\ref{alg:One} and prove Theorems~\ref{thm:lip_ext} and \ref{thm:low-dim-lip_ext}
via the multiplicative update framework of
\citet{Arora2012the}.
In particular, we will invoke their Theorem 3.4, which, for completeness, is reproduced in
Section~\ref{sec:AHK} as Theorem~\ref{thm:AHK}.
%\armin{SOMETHING NEEDS TO BE FIXED HERE: We start x
To simplify the notation, we assume (without loss of generality) that $L:=\lip{f}=1$.
Let $\calP = \Ball(y^{\circ}, \|x^{\circ} - {x^*}\|)$, and define $h_i(y) = 1 - \frac{\norm{y-y_i}}{\norm{x^* - x_i}}$ for $y\in \Y$ $i\in\{1,\dots,n\}$.
Then the Lipschitz extension problem is equivalent to the following:
find  $y\in\calP$  such that $h_i(y) \geq 0$ for all $i\in[n]$.
Note that functions $h_i$ are concave and thus the problem is in the form of (3.8) from~\citet{Arora2012the}.
We now bound the ``width'' of the problem, proving that $h_i(y) \in [-2,1]$ for every $y\in\calP$
(in the notation from~\citet{Arora2012the}, we show that $\ell\leq 1$ and $\rho \leq 2$). 
Observe that for every $y\in\calP$ and every $i$, we have (i) $h_i(y)\leq 1$ as $\frac{\|y - y_i\|}{\|{x^*} - x_i\|} \geq 0$ and (ii)
\beq
1 - h_i(y) &=& \frac{\|y - y_i\|}{\|{x^*} - x_i\|} \leq \frac{\|y - y^{\circ}\| + \|y^{\circ} - y_i\|}{\|{x^*} - x_i\|}\\
&\leq& \frac{\|x^{\circ} - {x^*}\| + \|x^{\circ} - x_i\|}{\|{x^*} - x_i\|}\\
&\leq& \frac{2\|x^{\circ} - {x^*}\| + \|{x^*} - x_i\|}{\|{x^*} - x_i\|}\\
&=& 1 + 2 \frac{\|x^{\circ} - {x^*}\|}{\|{x^*} - x_i\|} \leq 3.
\eeq
Here, we used that
$\|y - y^{\circ}\| \leq \|x^{\circ} - {x^*}\|$ (which is true since $y\in\calP$),
$\|y^{\circ} - y_i\|\leq \|x^{\circ} - x_i\|$ (which is true since $f$ is 1-Lipschitz), and
$\|{x^*} - x^{\circ}\| \leq \|{x^*} - x_i\|$ (which is true since $x^{\circ}$ is the point closest to ${x^*}$ among all points $x_1,\dots, x_n$).
We conclude that $h_i(y) \in [-2,1]$.

To apply Theorem~\ref{thm:AHK},
we design an oracle for the following problem:
\begin{problem}\label{prob:oracle}
Given non-negative weights $w_i$ that add up to $1$, find $y\in \calP$ such that
\begin{equation}
\sum_{i=1}^n w_i h_i(y) \geq 0. \label{eq:oracle}
\end{equation}
\end{problem}
\noindent Note that Problem~\ref{prob:oracle} has a solution, since ${y^*}$, the Lipschitz extension of $f$ to ${x^*}$ (whose existence is guaranteed by the Kirzsbraun theorem), satisfies (\ref{eq:oracle}).
Define auxiliary weights $p_i$ and $q_i$ as follows:
$$P = \sum_{i=1}^n \frac{w_i}{\|{x^*} - x_i\|^2}\quad\text{and}\quad p_i = \frac{w_i}{P\|{x^*} - x_i\|^2},$$
$$Q = \sum_{i=1}^n \frac{w_i}{\|{x^*} - x_i\|} \quad\text{and}\quad q_i = \frac{w_i}{Q\|{x^*} - x_i\|}.$$
The oracle finds and outputs $z\in \calP$ that minimizes $V(z) = \sum_{i=1}^n p_i \|z-y_i\|^2$. To this end, it first computes
$z_0 = \sum_{i=1}^n p_i y_i$. Note that
$V(z) = \|z-z_0\|^2 + \sum_{i=1}^n p_i\|z_0 - y_i\|^2.$
Then, if $z_0\in\calP$, it sets $z = z_0$; otherwise, $z$ is set to be the point closest to $z_i$ in $\calP$, which is
$$z = \frac{\|x^{\circ} - {x^*}\|}{\|z_0 - y^{\circ}\|} z_0 + \left(1-\frac{\|x^{\circ} - {x^*}\|}{\|z_0 - y^{\circ}\|}\right) y^{\circ}.$$
This $z$ is computed on lines 6--8 of the algorithm. We verify that $z$ satisfies condition (\ref{eq:oracle}).
Rewrite condition (\ref{eq:oracle}) in terms of weights $q_i$:
$Q\sum_{i=1}^n q_i \,\|y-y_i\| \leq 1$.
Using that
$$\|{y^*}-y_i\|^2\left/\|{x^*}-x_i\|^2\right.\leq 1,$$ we get
\begin{align*}
Q\sum_{i=1}^n q_i \|z-y_i\| &= \sum_{i=1}^n w_i \frac{\|z-y_i\|}{\|{x^*}-x_i\|}\\
&{\leq} \left(\sum_{i=1}^n w_i \frac{\|z-y_i\|^2}{\|{x^*}-x_i\|^2}
\sum_{i=1}^n w_i\right)^{1/2} \\
&{\leq}
\left(\sum_{i=1}^n w_i \frac{\|{y^*}-y_i\|^2}{\|{x^*}-x_i\|^2}\right)^{1/2} \leq 1
.
\end{align*}
The first inequality is due to Cauchy--Schwarz, and the second holds since $ V(z)\leq V(y^*)$.
\begin{algorithm}
\caption{\OnePt}
\label{alg:One}
\begin{algorithmic}[1]
  \Require
  labeled sample $(x_i,y_i)\subset(X\times Y)^n$,
$\eps \in (0,1/2)$
  query point $x^*\in X$,
  and upper bound $L\ge\lip{f}$
\Return label $y^*$
\State \textbf{let } $x^{\circ}$ be the nearest neighbor %
of ${x^*}$ among $x_1,\dots, x_n$; $y^{\circ} = f(x^{\circ})$; $d^{\circ} = \|x^{\circ} - {x^*}\|$
\State initialize weights $w_1^{(1)},\dots, w^{(1)}_n$ as follows: $w^{(1)}_i = 1/n$ for every $i$
\State \textbf{let} $d_i = \|x_i - {x^*}\| / L$ for every $i$
\State \textbf{let} $T = \lceil\frac{16 \ln n}{\eps^2}\rceil$ (the number of iterations)

\For {$t = 1$ to $T$}
\State \textbf{let} $P = \sum_{i=1}^n w_i^{(t)}/d_i^2$ and $p_i = \frac{w_i^{(t)}}{P d_i^2}$.
\State \textbf{let} $z_0 = \sum_{i=1}^n p_i y_i$ and $\Delta = \|z_0 - y^{\circ}\|$
\State \textbf{if} $\Delta \leq d^{\circ}$ \textbf{then} $z^{(t)} = z_0$ \textbf{else} $ z^{(t)} = \frac{d^{\circ}}{\Delta} z_0 + \frac{\Delta - d^{\circ}}{\Delta} y^{\circ}$
\State \textbf{update the weights:} $w_i^{(t+1)} = (1 + \frac{\eps \|{\tilde y}-y_i\|}{8d_i})w_i^{(t)}$ for every $i$
\State \textbf{normalize the weights: } compute $W = \sum_{i=1}^n w_i^{(t)}$ and let $w_i^{(t+1)} = w_i^{(t)}/W$ for every $i$
\EndFor

\State \textbf{return} $z = \frac{1}{T}\sum_{t=1}^T z^{(t)}$

\end{algorithmic}
\end{algorithm}

%\armin{BRACKETS}
\begin{proof}{\it Proof of Theorem \ref{thm:lip_ext}.}
From Theorem~\ref{thm:AHK},
  we get that the algorithm finds a $1 + \eps$ approximate
solution in $T = \frac{8\rho\ell \ln m}{\eps^2} = \frac{16 \ln m}{\eps^2}$ iterations. Computing distances $d_i$ takes $O(a n)$ time,
each iteration takes $O(b n)$ time.
\qed
\end{proof}

\begin{proof}{\it 
Proof of Theorem \ref{thm:low-dim-lip_ext} (sketch).}
Our key observation is that we can run the algorithm from Theorem~\ref{thm:lip_ext} on a subset $X'$ of $X$, which is
sufficiently dense in $X$. Specifically, let $x^{\circ}$ be a $(1+\eps)$-approximate nearest neighbour for $x^*$ in $X$.
Assume that a subset $X' \subset X$ contains $x^{\circ}$ and satisfies the following property: for every $x_i \in X \cap \Ball(x^*, \|x^* - x^{\circ}\|/\eps)$,
there exists $x_j \in X'$ such that $\|x_j - x_i\| \leq \eps \|x^* - x_i\|$.

First, we will prove  that by running the algorithm on set $X'$ we get $y^*$ such that $\|y_i - y^*\| \leq (1+O(\eps)) L \|x_i - x^*\|$ for all $i$.
Then we describe a data structure that we use to find $X'$ for a given query point $x^*$ in time $(1/\eps)^{O(a)} \log n$.

\noindent (1) Algorithm from Theorem~\ref{thm:lip_ext} finds $y^*$ such that $\|y_i - y^*\| \leq (1+O(\eps)) L \|x_i - x^*\|$ for all $x_i \in X'$.
Consider $x_i \in X$. First, assume that $x_i \in \Ball(x^*, \|x^* - x^{\circ}\|/\eps)$. Find $x_j \in X'$ such that $\|x_j - x_i\| \leq \eps \|x^* - x_i\|$.
Then
\begin{align*}
\|y_i - y^*\| &\leq \|y_i - y_j\| + \|y_j - y^*\|\\
& \leq L \|x_i - x_j\| + (1+\eps) L \|x_j - x^*\|\\
&\leq L((2 + \eps)\|x_i - x_j\| + (1 + \eps)\|x_i - x^*\|)\\
& \leq (1 + 3\eps + \eps^2)M \|x_i - x^*\|,
\end{align*}
as required. Now assume that $x_i \notin \Ball(x^*, \|x^* - x^{\circ}\|/\eps)$.
\begin{align*}
\|y_i - y^*\| &\leq \|y_i - y^{\circ}\| + \|y^{\circ} - y^*\| \\
&\leq L \|x_i - x^{\circ}\| + (1+\eps) L \|x^{\circ} - x^*\|\\
&\leq L(\|x_i - x^*\| + (2 + \eps)\|x^\circ - x^*\|)\\
& \leq (1 + \eps)^2 L \|x_i - x^*\|.
\end{align*}

We use a data structure $\cal D$ for approximate nearest neighbor search in $X$ . We employ one of the constructions for low-dimensional Euclidean spaces,
by either of
\citet{arya1994optimal} or \citet{har-peled2006fast}.
Using $\cal D$, we can find a $(1+\eps/3)$-approximate nearest neighbor of a point in $\bbR^a$ in time $(1/\eps)^{O(a)} \log n$.
Recall that we can construct $\cal D$ in $O(2^{O(a)} n \log n)$ time, and it requires $O(2^{O(a)} n\log n)$ space.
Suppose that we get a query point $x^*$. We first find an approximate nearest neighbor $x^{\circ}$ for $x^*$. Let $r = \|x^{\circ} - x^*\|$.
Take an $\eps r/3$ net $N'$ in the ball $\Ball(x^*, r/\eps)$.
For every point $p \in N'$, we find an approximate nearest neighbor $x(p)$ in $X$ (using $\cal D$). Let $X' = \{x(p): p \in N'\} \cup \{x^{\circ}\}$.
Consider $x_i \in \Ball(x^*, r/\eps)$. There is $p\in X'$ at distance at most $\eps r/3$ from $x_i$.
Let $x_j = x(p) \in X'$. Then
$$\|p - x_j\| \leq (1+\eps/3) \|x_i - p\| \leq (1+\eps/3) \eps r/3$$ and
\begin{align*}
\|x_i - x_j\| &\leq \|x_i - p\| + \|p-x_j\| \\
&\leq \eps r/3 + (1+\eps/3) \eps r/3 \\
&\leq (2+\eps/3) \eps r/3 \leq 3 \|x^* - x_i\|,
\end{align*}
as required.
The size of $X'$ is at most the size of $N'$, which is $(1/\eps)^{O(a)}$.
\qed
\end{proof}

 \begin{algorithm}
\caption{\MultiPt}
\label{alg:Multi}
\begin{algorithmic}[1]
\Require vectors $x_1,\dots, x_n, x_{n+1},\dots,$
$ x_{n+n'} \in \bbR^a$ and $y_1,\dots, y_n \in \bbR^b$, graph $G = ([n],
E)$, and $M$
\Return $\widetilde{\mathbf{Y}} = (y_{n+1},\dots, y_{n+n'})$

\State \textbf{let} ${\mathbf Y} \equiv (y_1,\dots, y_n)$
\State \textbf{let} $r_{ij} = L \|x_i - x_j\|$ for $(i,j) \in E$
\State \textbf{let} $w_{ij} = 1/m$ for $(i,j)\in E$
\State \textbf{let} $T = c_1 \lceil\frac{\sqrt{m} \ln n}{\varepsilon^2}\rceil$ (the number of iterations)
\For{\textbf{for} $t = 1$ to $T$}
\State \textbf{let} $\lambda_{ij} = \frac{w_{ij} + \varepsilon/m}{r_{ij}^2}$ for $(i,j) \in E$
\State \textbf{define} $n' \times n$ times matrix $K$ as:\\ 
$K_{ij} = \lambda_{i+n, j+n}$ if $(i+n, j+n)\in E$ and $0$ otherwise.
\State \textbf{solve} $L ({\widetilde{\mathbf{Y}}}^t)^\top = K{\mathbf Y}^\top$ for $\widetilde{\mathbf{Y}}^{t}$
\State \textbf{update the weights:} $w_{ij}= \left(1 + c_2 \varepsilon \left(\frac{\|y_i - y_j\|}{r_{ij}} -1\right)\right) w_{ij}$ for every $(i, j) \in E$
\State \textbf{normalize the weights:}  $W = \sum_{(i,j)\in E} w_{ij}$\\ 
$w_{ij} = w_{ij}/W$ for all $(i,j) \in E$
\EndFor\\
\Return $\widetilde{\mathbf{Y}} = \frac{1}{T}\sum_{t=1}^{T} \widetilde{\mathbf{Y}}^t$

\end{algorithmic}
\end{algorithm}

\paragraph{Multi-point Lipschitz extension.}
Finally, we describe an algorithm for the Multi-point Lipschitz Extension. The problem is a generalization of the problem we studied in Section~\ref{sec:lipext} 
%COMMENT: MAYBE YOU SHOULD MAKE A NEW SECTION HERE AS IT IS STILL PART OF SECTION 3.1, MAKE A CLEARER DIVISION BETWEEN OnePointExt AND MultiPointExt}. 
We are given a set of points $X = \{x_1,\dots, x_n\}\subset \bbR^a$ and their images $Y = \{y_1,\dots, y_n\}\subset \bbR^b$ under $L$-Lipschitz map $f$. Additionally, we are given a set $Z = \{x_{n+1},\dots, x_{n+n'}\}\subset \bbR^a$ and a set of edges $E$ on $\{1,\dots, n + n'\}$.
We need to extend $f$ to $Z$ --- that is, find $y_{n+1},\dots, y_{n+n'}$ --- such that $\|y_i - y_j\| \leq (1+\eps) L \|x_i - x_j\|$ for $(i,j) \in E$. We note that $E$ may contain edges that impose Lipschitz constraints
(i) between points in $X$ and $Z$ and (ii) between pairs of points in $Z$. Without loss of generality, we assume that there are no edges $(i,j) \in E$ with $1\leq i,j\leq n$.
\begin{theorem}
  There is an algorithm for the Multi-point Lipschitz Extension problem that runs in time
  $$O(ma + \frac{m^{3/2} (\log m)^2 b \log (1/\eps)}{\eps^{5/2}}),$$ where $m = |E|$.
\end{theorem}

The algorithm and its analysis are almost identical to those for the Lipschitz Smoothing problem. (see Theorem~{\ref{thm:LipSmooth}}).

\subsection{
%ERM and Lipschitz 
Smoothing}
\label{sec:ERM}
\paragraph{Problem statement.}
We reformulate the ERM problem
$\hat f=\argmin_{f\in F_L}\hat R_n(f)$
as follows.
Given two sets of vectors, $(x_i,y_i)_{i\in[n]}$, where $x_i\in X:=\R^a$ and $y_i\in Y:=\R^b$,
we wish to compute a ``smoothed'' version $\tilde y_i$ of the $y_i$'s
so as to
\begin{align*}
\mathrm{Minimize}\ & \Phi(\mathbf{Y}, \widetilde{\mathbf{Y}}) := \sum_{i=1}^n \|y_i - \tilde y_i\|^2\\
\mathrm{subject}\ to\ \ & \|\tilde y_i - \tilde y_j\| \leq 
L \|x_i - x_j\|,   i,j\in[n]
\end{align*}
$\Phi(\mathbf{Y}, \widetilde{\mathbf{Y}}) := \sum_{i=1}^n \|y_i - \tilde y_i\|^2$ is the distortion,
and 
$\|\tilde y_i - \tilde y_j\| \leq 
L %M
\|x_i - x_j\|$ for all $i,j\in[n]$ are the Lipschitz constraints.
Here, $\mathbf{Y} = (y_1,\dots, y_n)$ and $\widetilde{\mathbf{Y}} = (\tilde
y_1,\dots, \tilde y_n)$ (the columns of matrices $\mathbf Y$ and $\widetilde{\mathbf{Y}}$ are vectors $y_1,\dots, y_n$ and $\tilde y_1,\dots, \tilde y_n$, respectively).
Notice that when we use the $L_2$ norm, this
problem is a quadratically constrained quadratic program (QCQP). 
%(Notice that for this section we always take $M=\lip{f}$.)

We consider a more general variant of this problem where we are given a set of
edges $E$ on $\{1,\dots, n\}$, and the goal is to ensure that the Lipschitz constraints
$\|\tilde y_i - \tilde y_j\| \leq 
L% M
\|x_i - x_j\|$ hold (only) for $(i,j) \in E$.
The original problem corresponds to the case when $E$ is the complete graph,
($E_{ij} = L \norm{x_i - x_j}$).
Importantly, if the doubling dimension $\operatorname{ddim}
X$ is low, we can solve the original problem by letting $([n], E)$ be
a $(1+\eps)$-stretch spanner; then $m= n (1/\eps)^{O(\operatorname{ddim})}$ (this approach was previously used by
\citet{gottlieb2017efficient}; see also~\citet[,Section 8.2]{har-peled2006fast},
who used a similar approach to compute the doubling constant).
Our algorithm for
Lipschitz Smoothing iteratively solves Laplace's problem in the graph $G$.
We proceed to
define this problem and present a closed-form formula for
the solution.
\paragraph{Laplace's problem.}\label{LaplaceProblem} We are given vectors $\{y_i\}$, graph $G$, and additionally vertex weights $\lambda_i \geq 0$ (for $i\in[n]$) and edge
weights $\mu_{ij} \geq 0$ (for $(i,j)\in E$), find $\tilde y_i$ so as to
\begin{align*}
\mathrm{minimize}\quad \Psi(\mathbf{Y}, \widetilde{\mathbf{Y}},
\{\lambda_i\}, \{\mu_{ij}\})\equiv \\
\sum_{i=1}^n \lambda_i \|y_i - \tilde
y_i\|^2 + \sum_{(i,j)\in E}\mu_{ij} \|\tilde y_i - \tilde y_j\|^2
.
\end{align*}

\newcommand{\Lap}   {\mathcal{L}}

Let $\Lap$ be the Laplacian of $G = ([n], E)$ with edge weights $\mu_{ij}$; that is $L_{ii} = \sum_{j:j\neq i} \mu_{ij}$
and $L_{ij} = -\mu_{ij}$ for $i\neq j$. Let $\Lambda = \operatorname{diag}(\lambda_1,\dots,\lambda_n)$. Then

$$(\Lap + \Lambda)\, {\widetilde{\mathbf{Y}}}^\top = \Lambda{\mathbf Y}^\top.$$
This equation can be solved separately for each of $b$ rows of ${\mathbf
Y}^\top$ using an nearly-linear equation solver for diagonally dominant matrices by
\citet{koutis2012fast} in total time $O(b m \log n \log (1/\eps))$
(see also the paper by
\citet{spielman2004nearly}, which presented
the first nearly-linear time solve for diagonally dominant matrices).

We
solve the Lipschitz Smoothing problem
via the
multiplicative weight update
algorithm \LipSmooth, presented below.
It was inspired
by the algorithm
for finding maximum flow using electrical networks by
\citet{christiano2011electrical}.

\begin{algorithm}
\caption{\LipSmooth}
\label{alg:Smooth}
\begin{algorithmic}[1]

\Require vectors $\mathbf{X}  = x_1,\dots, x_n \in {\mathbb R}^a$, $\mathbf{Y} = y_1,\dots, y_n \in {\mathbb R}^b$, graph $G = ([n],
E)$, $M$ and $\Phi_0$ 
\newline
%\noindent
\textbf{Output:} 
%\Return
$\widetilde{\mathbf{Y}}$\\
\State \textbf{let } ${\mathbf Y} \equiv (y_1,\dots, y_n)$
\State \textbf{let} $r_{ij} = L \|x_i - x_j\|$ for $(i,j) \in E$
\State \textbf{let} $w_{ij} = 1/(m+1)$ for $(i,j)\in E$, where $m=|E|$\\
\State \textbf{let} $w_{\Phi} = 1/(m+1)$
\State \textbf{let} $T = c_1 \lceil\frac{\sqrt{m} \ln n}{\varepsilon^2}\rceil$ (the number of iterations)
\For{ \textbf{for} $t = 1$ to $T$}
\State \textbf{let} $L$ be the Laplacian of $G$ with edge weights $\mu_{ij} = \frac{w_{ij} + \varepsilon/(m+1)}{r_{ij}^2}$
\State \textbf{let} $\lambda = (w_{\Phi} + \varepsilon/(m+1))/\Phi_0$
\State \textbf{solve} $(\lambda^{-1} \Lap + I)\, ({\widetilde{\mathbf{Y}}}^t)^\top = {\mathbf Y}^\top$ for $\widetilde{\mathbf{Y}}^{t}$
\State \textbf{update the weights:}
         $w_{\Phi}= \left(1 + c_2 \varepsilon \left(\sqrt{\frac{\Phi(\mathbf{Y}, \widetilde{\mathbf{Y}}^t)}{\Phi_0}} -1\right)\right) w_{\Phi}$ \\ \phantom{\textbf{update the weights:}}
         $w_{ij}= \left(1 + c_2 \varepsilon \left(\frac{\|\tilde y_i- \tilde y_j\|}{r_{ij}} -1\right)\right) w_{ij}$ for every $(i, j) \in E$
\State \textbf{normalize the weights:} $W = w_{\Phi} + \sum_{(i,j)\in E} w_{ij}$ \\
\phantom{\textbf{normalize the weigh}} $w_{\Phi}= w_{\Phi}/W$\\ \phantom{\textbf{normalize the weigh}} $w_{ij} = w_{ij}/W$ for all $(i,j) \in E$\\
\EndFor
\State \textbf{return} $\widetilde{\mathbf{Y}} = \frac{1}{T}\sum_{t=1}^{T} \widetilde{\mathbf{Y}}^t$

\end{algorithmic}
\end{algorithm}

\paragraph{Analysis.}
Let $\mathbf{Y^*}$ be the optimal solution to the Lipschitz Smoothing problem and 
and $\Phi_0$ be a $(1+ \varepsilon)$ approximation to the optimal value; that is, 
$$\Phi(\mathbf{Y}, {\mathbf{Y^*}}) \leq  \Phi_0 \leq (1 + \varepsilon) \Phi(\mathbf{Y}, {\mathbf{Y^*}})$$
(we assume that $\Phi_0$ is given to the algorithm; note that $\Phi_0$ can be found by binary search).

As in Section~\ref{sec:lipext}, we use the multiplicative-weight update (MWU) method.
Let
\begin{align*}
  h_{\Phi}(\widetilde{\mathbf{Y}}) &= 1 - \sqrt{\frac{\Phi(\mathbf{Y}, \widetilde{\mathbf{Y}})}{\Phi_0}},
  \quad \\
h_{ij}(\widetilde{\mathbf{Y}}) &= 1 - \frac{\|\tilde y_i - \tilde
  y_j\|}{M\|x_i - x_j\|}, \qquad
(i,j)\in E.
\end{align*}
Note that functions $h_{\Phi}$ and $h_{ij}$ are concave.

Observe that
$h_{\Phi}(\widetilde{\mathbf{Y}}^*) \geq 0$
 and
$h_{ij}(\widetilde{\mathbf{Y}}^*) \geq 0$ for every $(i,j) \in E$. On the other
hand, if $h_{\Phi}(\widetilde{\mathbf{Y}}) \geq -\eps$ and
$h_{ij}(\widetilde{\mathbf{Y}}) \geq -\eps$, then
$$\Phi(\mathbf{Y}, \widetilde{\mathbf{Y}}) \leq (1+\eps)^2 \Phi_0$$ and
$\|\tilde y_i - \tilde y_j\| \leq (1+\eps) L \|x_i - x_j\|$ for every
$(i,j)\in E$.

In the Appendix, we describe the approximation oracle that we invoke in the MWU method.

\begin{theorem}\label{thm:LipSmooth}
There is an algorithm for the Lipschitz Smoothing problem
that runs in time
$$O(ma + m^{3/2}b (\log n)^2 \log (1/\varepsilon)/\eps^{5/2}),$$ where $m=|E|$.
\end{theorem}
\begin{proof}
{\it Proof of Theorem \ref{thm:LipSmooth}.}
From Theorem 3.5 in~\citet{Arora2012the}, we get that the algorithm finds an $O(\eps)$ approximate
solution in
$T = O\left(\frac{\sqrt{m/\eps} \ln m}{\eps^2}\right) = \left(\frac{\sqrt{m}}{\eps^{5/2}}\right)$
iterations. Each iteration takes $O(b m \log n \log (1/\eps))$ time (which is dominated by the time necessary to solve Laplace's problem); additionally, we spend time $O(am)$ to compute pairwise distances between points in $X$.
\qed
\end{proof}

\section{Experiments}
\label{sec:experiments}

To illustrate the utility of our framework, we designed two simple non-linear transformation problems where the input and output are both scalars. Our data was generated uniformly at random over $[-2\pi,2\pi]$ and evaluated the performance on two cases: $f(x) = x^3$ and $f(x) = sin(x)$. 
\paragraph{Results.} In order to perfrom a fair, apples-to-apples comparison, we implemented both Algorithms ~\ref{alg:Smooth} and \ref{alg:One} in Matlab, which standard, optimzied QCQP solvers, and performed the regression problem via the Kirszbraun extension technique. We compared the results of this learning method when using our methods for the optimization problems (MWU) vs using Matlab's QCQP solver based on the interior-point algorithm (IntPt). We considered the squared Euclidean distance as the loss function. We ran several tests using different size data sets of 20, 100, 200, 500, and 1000 random points as training set, and 100 test points in all experiments. For reproducibility, we've used Matlab's random seed 1 in all our runs. All the tests where conducted on the same Macbook pro computer. The numeric comparison (Tables 1-5) shows undoubtedly supremacy of the MWU over the IntPt method both in efficiency and better learning. MWU method is able to optimize a data set of several thousands data points, while the IntPt based method could not complete its process in ``reasonable'' time (over 10 hours night run) with more than $N=200$ training points. In terms of solving the learning problem, the MWU able to solve the QCQP problem and produce accurate smoothing and more accurate extensions functions as the data size grows. The IntPt method, on the other hand, able to meet all the constrains of the problem only with very small data set (less then 50 training points) which is insufficient data to for learning. Tables 1-5 shows that for 20 training points, the IntPt is able to train in 2.474 seconds and completely over fit the data set with 0 ERM, which leads to expected very poor generalization due to the size of the data (Table 5).  On larger datasets the IntPt optimization fails to correctly solve the optimization problems with respect to all of the constrains. This result in several ``heavy'' outliers which affect heavily on the average square error of both smoothing and extension phases as can be seen in tables 1-5. Table 6 shows a graphical comparison for both implementations when the training set has $N=100$ points. The ``heavy outliers'' can be spotted easily on the graph, and explain why the same learning algorithm has such big differences when optimised with two different methods.

Tables 1-6 summarise the results for $f(x) = x^3$. The results for $f(x) = \sin{x}$ are showing the same basic pattern and were added to the appendix. 
%\begin{comment}
%\armin{WHY TO MOVE TO THE APPENDIX?I THINK ALL EXPERIMENTS SHOULD BE IN THE %MAIN TEXTS}
%\AK{Since MOR is theoretically oriented, putting too many experiments in the main text might distract from the more relevant contributions. I suggest keeping those in the Appendix. Yury?}
%\end{comment}
The blank entries in the tables indicate that the process did not terminate in the time allotted (12 hours).\\
%\armin{(1) THE MATLAB RESULTS FOR $n=200$ SEEM TO BE BETTER. ANY COMMENTS NEEDED?}\\ -> corrected
% \armin{(2) A DISCUSSION AND CONCLUSION SECTION IS MISSING.} -> added
% \armin{(3) TABLES: I WOULD PUT THE 'NAME' OF THE ALGORITHM IN THE FIRST ROW FOLLOWED BY 'TRAINING POINTS' ETC}

% \armin{(4) WE CLAIM TO HAVE A NON-PARAMETRIC REGRESSION METHOD THAT IS VECTOR-VALUED. EXAMPLES ARE ALL ONE-DIMENSIONAL. COMMENTS ON THAT?}\\ -> it was more complex to create high dimensional problem that can also be graphical and suitable for our problem. In any case our experiments focus on the optimization improvement thus any qcqp problem support our case.
% \armin{(5) SHOULD EQUATIONS BE NUMBERED THROUGHOUT THE MANUSCRIPT?} -> we numbered only the equations that we referenced to

\begin{table}[h!]
  \begin{center}
    \caption{ERM of the smoothing process}
%    \label{tab:sm}
    \begin{tabular}{l|c|c|c|c|c} %
     training points & 20 & 100 & 200 & 500 & 1000 \\ \hline
     \textbf{Algorithm} & \multicolumn{5}{c}{\textbf{Avg. loss}} \\
      \hline
      MWU & 247.9405 & 0.3333 & 0.31581 & 0.31854 & 0.36143\\
      IntPt & 4.1e-18 & 46023.7964 & 353691.64 & & \\
  
    \end{tabular}
  \end{center}
\end{table}

\begin{table}[h!]
  \begin{center}
    \caption{Cross validation running time over the smoothing process in seconds}
%    \label{tab:res}
    \begin{tabular}{l|c|c|c|c|c} %
     training points & 20 & 100 & 200 & 500 & 1000 \\ \hline
     \textbf{Algorithm} & \multicolumn{5}{c}{\textbf{Avg. loss}} \\
      \hline
      MWU & 2.7505 & 19.9428 & 46.2479 & 212.4875 & 1243.2395\\
      IntPt & 18.1733 & 692.6523 & 4087.6655 & & \\
  
    \end{tabular}
  \end{center}
\end{table}

\begin{table}[h!]
  \begin{center}
    \caption{Single$^*$ smoothing process running time in seconds}
%    \label{tab:res}
    \begin{tabular}{l|c|c|c|c|c} %
     training points & 20 & 100 & 200 & 500 & 1000 \\ \hline
     \textbf{Algorithm} & \multicolumn{5}{c}{\textbf{Avg. loss}} \\
      \hline
      MWU & 0.092 & 0.7051  & 1.6326 & 8.0798 & 45.1497 \\
      IntPt & 2.474 & 155.5521  & 766.5433 & & \\
  
    \end{tabular}
  \end{center}
\end{table}

\begin{table}[h!]
  \begin{center}
    \caption{Extension avg loss}
%    \label{tab:res}
    \begin{tabular}{l|c|c|c|c|c} %
     training points & 100 & 200 & 500 & 1000 \\ \hline
     \textbf{Algorithm} & \multicolumn{5}{c}{\textbf{Avg. loss}} \\
      \hline
      MWU & 1119.4705 & 0.3333 & 0.37267 & 0.43678 & 0.52797\\
      IntPt & 3065.5698 & 9475.260 & 9475.4864 & & \\
  
    \end{tabular}
  \end{center}
\end{table}

\begin{table}[h!]
  \begin{center}
    \caption{Extension running time for all test points in seconds}
%    \label{tab:res}
    \begin{tabular}{l|c|c|c|c|r} %
     training points & 20 & 100 & 200 & 500 & 1000 \\ \hline
     \textbf{Algorithm} & \multicolumn{5}{c}{\textbf{Avg. loss}} \\
      \hline
      MWU & 0.054& 0.0014903  & 0.002826& 0.0051038  & 0.0089853 \\
      IntPt & 1.3367 & 1.6918  & 2.6156  & & \\
  
    \end{tabular}
  \end{center}
\end{table}

\begin{table}[h!]
\centering
 \caption{Visual comparison between our MWU based algorithm (first row) and IntPt (Matlab's) based algorithm (second row). For $f=x^3$ and $N=100$ random points. In All graphs the blue line represents the ground truth function $f=x^3$ while the orange 'x' symbols represent the estimation of the data points by the learned function. It is possible to see that while the MWU based algorithm was able to fit both the training and test set in high accuracy, the Intpt method has several ``heavy'' outliers which reduce significatly its average squared error}%    \label{tab:res}
    \begin{tabular}{l|c|c|} %
      & \textbf{Smoothing} & \textbf{Extension} \\ \hline
     \centering\textbf{MWU} & \includegraphics[width=1.8in] {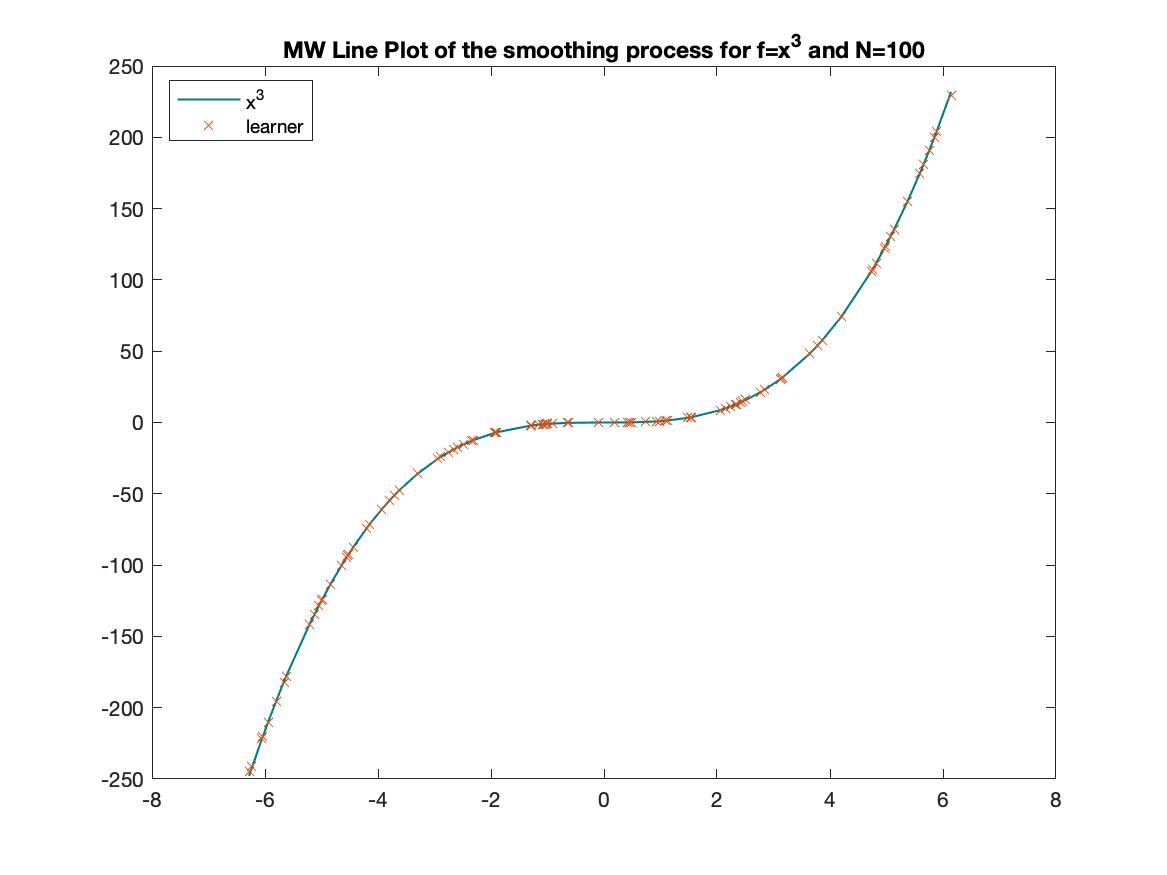}
    	 & \includegraphics[width=1.8in] {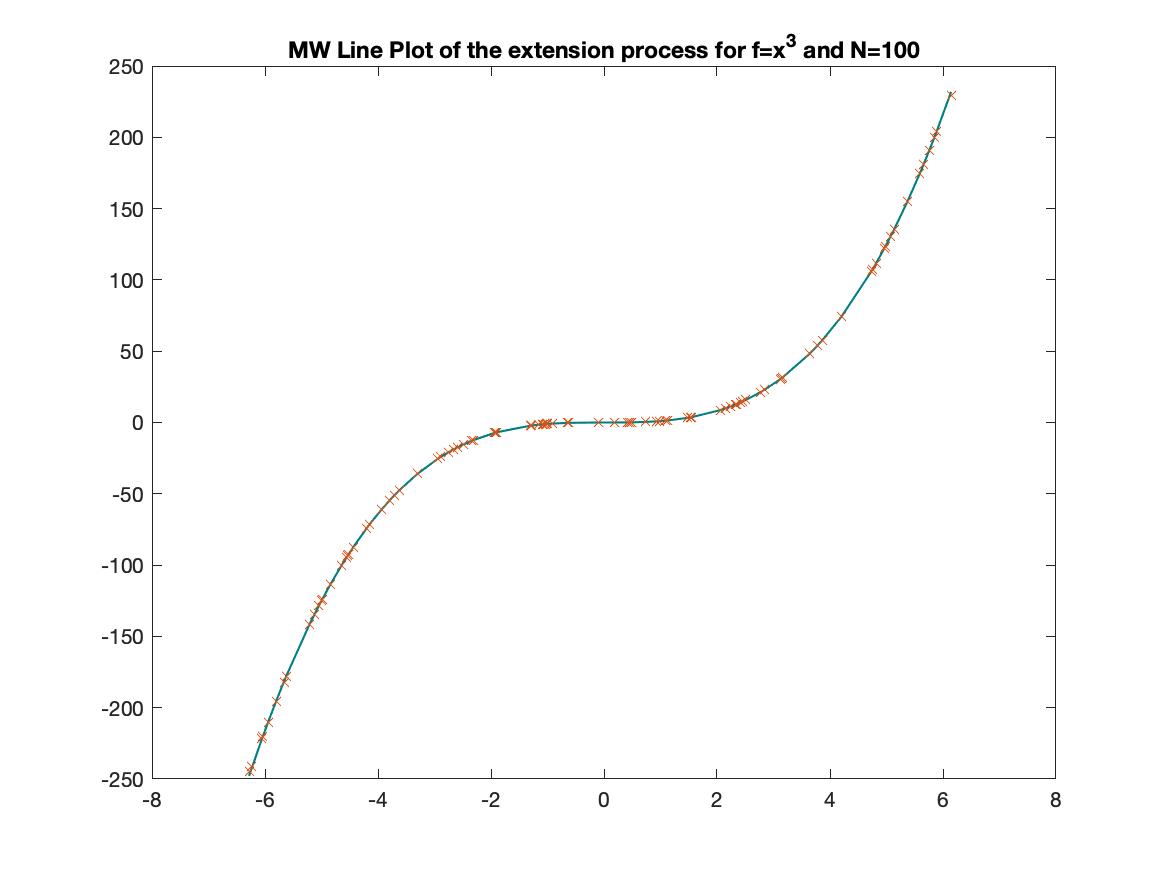} \\ \hline
    \textbf{IntPt} & \includegraphics[width=1.8in] {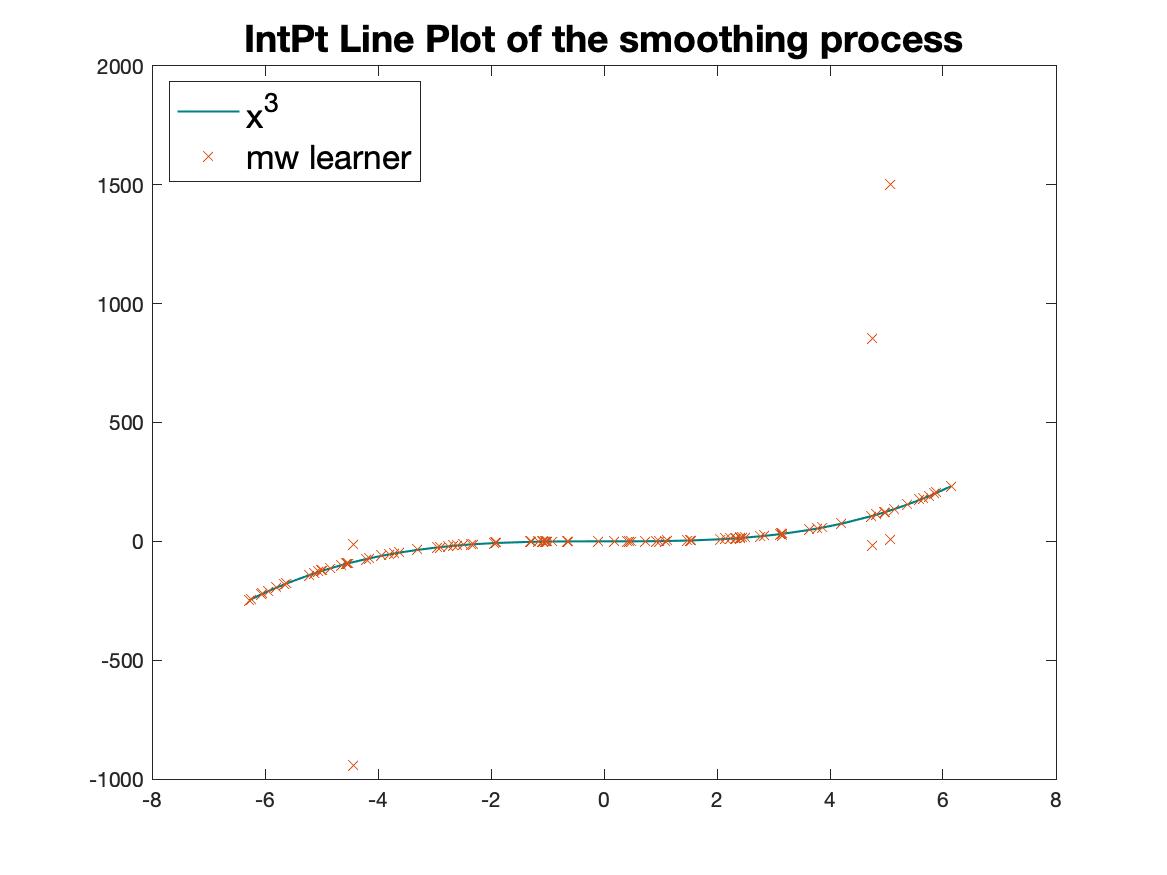}
    	 & \includegraphics[width=1.8in] {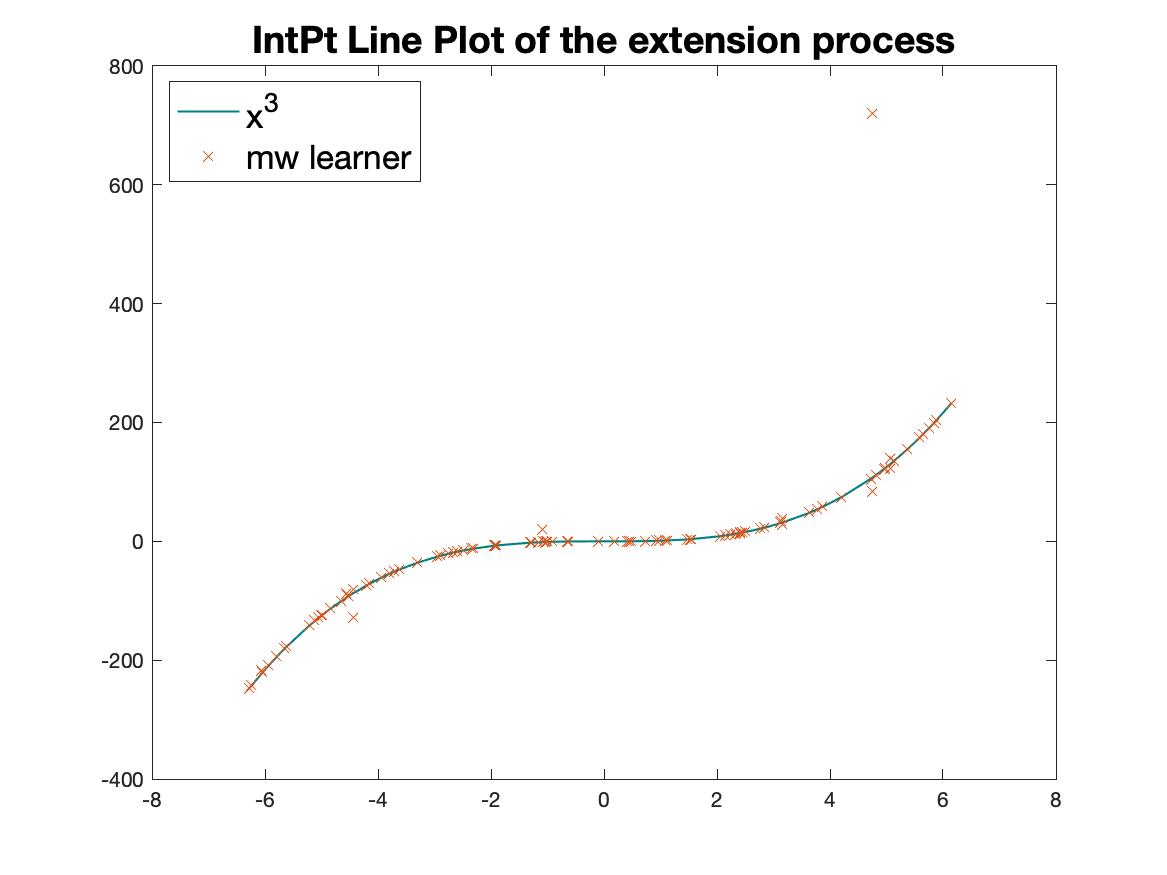}	\\ 
    \end{tabular}
\end{table}

\begin{comment}
\begin{figure}[h]
  \centering
  	\begin{subfigure}[b]{0.4\linewidth}
    	\includegraphics[width=\linewidth] {mw_100_x_3_smoothing.jpg}
    	\caption{}
  	\end{subfigure}
  	\begin{subfigure}[b]{0.4\linewidth}
    	\includegraphics[width=\linewidth] {mw_100_x_3_extension.jpg}
		\caption{}
  	\end{subfigure}
  	\begin{subfigure}[b]{0.4\linewidth}
    	\includegraphics[width=\linewidth] {100_x_3_Intpt_smoothing.jpg}
		\caption{}
  	\end{subfigure}\begin{subfigure}[b]{0.4\linewidth}
    	\includegraphics[width=\linewidth] {100_x_3_Intpt_extension.jpg}
		\caption{}
  	\end{subfigure}
  \caption{Figures (a)-(d) shows the comparison between our MWU based algorithm ((a)-(b)) and IntPt (Matlab's) based algorithm ((c)-(d)). for $f=x^3$ and $N=200$ random points. In All graphs the blue line represents the ground truth function $f=x^3$ while the orange 'x' symbols represent the estimation of the data points by the learned function 
  %\armin{FIGURES AND LEGENDS ARE  TOO SMALL. }
  }
  \label{fig:fig1}
\end{figure}
\end{comment}

\section{Discussion and Conclusions}
This work introduces a framework for performing regression between two Hilbert spaces based  on Kirszbraun's extension theorem, along with statistical analysis for this method. This task is decomposed into two stages: Smoothing (which corresponds to the training) and prediction (which achieved via Kirszbraun extension). Numerically solving our optimization problems has indicated a need for a more efficient solver for our optimization problems than off the shelf state-off-the-art solvers. We introduced two optimization algorithms, one for the smoothing problem and one for the extension, both are solved algorithmically via novel MWU schemes. Both analysis and experiments shows dramatically run time improvement for both optimization problems thus indicating that this algorithms are the main contribution off this work and are interest topic for future research on their own. Our code is also provided for reproducibility and to facilitate usage.

\paragraph{Acknowledgements.}
AK was partially supported by
the Israel Science Foundation
(grant No. 1602/19),
the Ben-Gurion University Data Science Research Center,
and an Amazon Research Award.
HZ was an MSc student at Ben-Gurion University of the Negev during part of this research.

\newpage
%\begin{APPENDICES}

\appendix

%\section*{Appendix}

\newcommand{\oracle}{\textup{{\textsf{Oracle}}}}
\newcommand{\x}{\mathbf{x}}

\section{The Arora-Hazan-Kale result}
\label{sec:AHK}
For completeness, we quote here verbatim (except for the numbering)
the relevant definitions and results from
\cite[Sec. 3.3.1, p. 137]{Arora2012the}.

Imagine that we have the following feasibility problem:
\begin{equation}
  \label{eq:ahk-feas}
\exists?\x \in \calP:\> \forall i\in [m]:\> f_i(\x) \geq 0,
\end{equation}
where $\calP \in \mathbb{R}^n$ is a convex domain, and for $i\in [m]$, $f_i:\calP \rightarrow \mathbb{R}$ are concave functions. We wish to satisfy this system approximately, up to an additive error of $\eps$. We assume the existence of an \oracle, which, when given a probability distribution
$\mathbf{p} = (p_1,p_2,\ldots,,p_m)$ solves the following feasibility problem:
\begin{equation}
    \label{eq:ahk-feas2}
\exists?\x \in \calP:\> \sum_i p_i f_i(\x) \geq 0.
\end{equation}
An \oracle\ is said to be called $(\ell,\rho)$-bounded if there is a fixed subset of constraints $I \subseteq [m]$ such that whenever it returns a feasible solution $\x$ to (\ref{eq:ahk-feas2}), all constraints $i\in I$ take values i the range $[-\ell,\rho]$ on the point $\x$, and all the rest take values in $[-\rho,\ell]$.

\begin{theorem}[Theorem 3.4 in \citet{Arora2012the}]
  \label{thm:AHK}
  Let $\eps>0$ be a given error parameter. Suppose there exists an $(\ell,\rho)$-bounded \oracle\ for the
  feasibility problem (\ref{eq:ahk-feas}).
  Assume the $\ell\geq \eps/2$. Then there is an algorithms which either solves the problem up to an additive error of $\eps$, or correctly concludes that the system s infeasible, making only $O(\ell\rho\log(m)/\eps^2)$ calls to the \oracle, with an additional processing time of $O(m)$ per call.
\end{theorem}

\section
{Approximate oracle}
To use the MWU method (see  Theorem 3.5 in~\citet{Arora2012the}), we design an approximate oracle for the following problem.

\begin{problem}\label{prob:oracle-LipSmooth}
Given non-negative edge weights $w_{\Phi}$ and  $w_{ij}$, which add up to 1, find $\widetilde{\mathbf{Y}}$ such that
\begin{equation}
w_{\Phi} h_{\Phi}(\widetilde{\mathbf{Y}}) + \sum_{(i,j)\in E} w_{(i,j)} h_{(i,j)}(\widetilde{\mathbf{Y}}) \geq - \eps. \label{eq:oracle-LipSmooth}
\end{equation}
\end{problem}

Let $\mu_{ij} = \frac{w_{ij} + \eps/(m+1)}{M^2\|x_i - x_j\|^2}$ and $\lambda_i = \lambda = (w_{\Phi} + \eps/(m+1))/\Phi_0$.
We solve Laplace's problem with parameters $\mu_{ij}$ and $\lambda_i$ (see
Section~\ref{LaplaceProblem} and Line 9 of the algorithm). We get a
matrix $\widetilde{\mathbf{Y}} =  (\tilde y_1,\dots, \tilde y_n)$ minimizing
$$\lambda\sum_{i=1}^n \|y_i - \tilde
y_i\|^2 + \sum_{(i,j)\in E}\mu_{ij} \|\tilde y_i - \tilde y_j\|^2.$$
Consider the optimal solution $\tilde y_1^*,\dots, \tilde y_n^*$ for Lipschitz
Smoothing. We have
%\begin{align}
\beq
\label{eq:u-bound-h}
\lambda  \sum_{i=1}^n \|y_i - \tilde y_i\|^2 +
\sum_{(i,j)\in E}\mu_{ij} \|\tilde y_i - \tilde y_j\|^2 
%\\ 
%\notag
&\leq& \lambda  \sum_{i=1}^n \|y_i - \tilde y_i^*\|^2 + \sum_{(i,j)\in E} \mu_{ij} \|\tilde y_i^* - \tilde
y_j^*\|^2
\\ 
%\notag
& \leq&  (w_{\Phi} + \eps/(m+1))\ + 
%\\ \notag
%&\>\> 
\sum_{(i,j)\in E} \left(w_{ij} + \frac{1}{m+1}\right)\frac{\|\tilde y_i^* - \tilde y_j^*\|}{M^2 \|x_i - x_j\|^2}
\\
&\leq&
\left(w_{\Phi}+\sum_{(i,j)\in E} w_{ij}\right) + \eps = 1 + \eps
.
\eeq
%\end{align}

We verify that $\widetilde{\mathbf{Y}}$ is a feasible solution for Problem~\ref{prob:oracle-LipSmooth}.
We have
%\begin{align*}
\beq
1 - \left(w_{\Phi} h_{\Phi}(\widetilde{\mathbf{Y}}) + \sum_{(i,j)\in E} w_{(i,j)} h_{(i,j)}(y)\right)
&=& 
w_{\Phi}(1 - h_{\Phi}) + \sum_{(i,j)\in E} w_{(i,j)} (1 - h_{(i,j)}(y))\\
&\leq&
\sqrt{w_{\Phi}(1 - h_{\Phi})^2 + \sum_{(i,j)\in E} w_{(i,j)} (1 - h_{(i,j)}(y))^2}
\\&=&
\sqrt{w_{\Phi} \frac{\Phi(\mathbf{Y}, \widetilde{\mathbf{Y}})}{\Phi_0} + \sum_{(i,j)\in E} w_{(i,j)}  \frac{\|\tilde y_i - \tilde
y_j\|^2}{M^2\|x_i - x_j\|^2}}
\\
&\leq&
\sqrt{\lambda \Phi(\mathbf{Y}, \widetilde{\mathbf{Y}})+ \sum_{(i,j)\in E} w_{(i,j)} \mu_{ij}\|\tilde y_i - \tilde
y_j\|^2} 
\\
&\leq& \sqrt{1+\eps} \leq 1 + \eps,
%\end{align*}
\eeq
as required.

Finally, we bound the width of the problem. We have $h_{\Phi}(\widetilde{\mathbf{Y}}) \leq 1$ and $h_{ij}(\widetilde{\mathbf{Y}}) \leq 1$.
Then, using (\ref{eq:u-bound-h}), we get
%\begin{align*}
\beq
(1 - h_{\Phi}(\widetilde{\mathbf{Y}}))^2 
= 
\frac{1}{\Phi_0}\sum_{i=1}^n \|y_i - \tilde y_i\|^2 \leq \frac{1+\eps}{\lambda\Phi_0}  
\leq  
\frac{(1+\eps)(m+1)}{\eps}
.
\eeq
%\end{align*}
Therefore, $-h_{\Phi}(\widetilde{\mathbf{Y}}) \leq O(\sqrt{m/\eps})$.

Similarly,
\beq
%\begin{align*}
(1 - h_{ij}(\widetilde{\mathbf{Y}}))^2 
= 
\frac{\|y_i - \tilde y_i\|^2}{M^2\|x_i - x_j\|^2} \leq \frac{1+\eps}{\mu_{ij} \cdot M^2\|x_i - x_j\|^2}
\leq \frac{(1+\eps)(m+1)}{\eps}
.
\eeq
%\end{align*}

Therefore, $-h_{ij}(\widetilde{\mathbf{Y}}) \leq O(\sqrt{m/\eps})$.

\section{Generalization bounds
}
\label{sec:gen-bds}
Let $ \X\subset\R^k $ and $\Y\subset\R^\ell$ be
the unit balls of their respective
Hilbert spaces (each endowed with
the $\ell_2$ norm $\norm{\cdot}$ and corresponding inner product)
and $\H_L\subset\Y^\X$ be the set of all $L$-Lipschitz mappings from $\X$ to
$\Y$. In particular, every $h\in\H_L$ satisfies
\beq
\norm{ h(x)-h(x')} \le L\norm{x-x'},
\qquad x,x'\in\X.
\eeq

Let $\F_L\subset\R^{\X\times\Y}$
be the {\em loss class} associated with $\H_L$:
\beq
%\begin{align*}
\F_L= \{ \X\times\Y\ni(x,y)
\mapsto f(x,y)=f_h(x,y)
:=\norm{ h(x)-y}; h\in\H_L \}
.
\eeq
%\end{align*}
In particular, every $f\in\F_L$ satisfies $0\le f\le 2$.

Our goal is to bound the Rademacher complexity of $\F_L$. We do this via a covering numbers approach.

The empirical \emph{Rademacher complexity} of a collection of functions $\F$ mapping some set ${Z_1,\dots,\Z_n}\subset\Z^n$
to $\R$ is defined by:
\begin{equation}
\hat{\cR}(\F;\Z)= \E \Big [ \sup_{f\in \F} \frac{1}{n} \sum_{i=1}^n \sigma_if(\Z_i)\Big ].
\end{equation}

Recall the relevance of Rademacher complexities to uniform deviation estimates for the risk functional
$R(\cdot)$
\citep[Theorem 3.1]{mohri2012foundations}:
for every $\delta>0$, with probability at least $1-\delta$, for each $h\in \H_L$:
\begin{equation}
R(h(z))\leq \hat R_n(h(z)) + 2\hat{\cR}_n(\hat{\F_L}) + 6 \sqrt{\frac{\ln (2/\delta)}{2n}}.
\end{equation}

Define $\Z=\X\times\Y$ and endow it with the norm
$\nrm{(x,y)}_\Z=\nrm{x}+\nrm{y}$;
note that $(\Z,\nrm{\cdot}_\Z)$ is a Banach but not a Hilbert space.
First, we observe that the functions in $\F_L$ are Lipschitz
under $\nrm{\cdot}_\Z$.
Indeed,
choose any $f=f_h\in\F_L$ and $x,x'\in\X$, $y,y'\in\Y$. Then
\beq
\abs{
f_h(x,y)-f_h(x',y')
}
&=&
\abs{
  \nrm{ h(x)-y}
  -
  \nrm{ h(x')-y'}
}\\
&\le&
\nrm{
  ( h(x)-y)
  -
  ( h(x')-y')
}\\
&\le&
\nrm{h(x)-h(x')}
+
\nrm{y-y'}\\
&\le&
L\nrm{x-x'}+\nrm{y-y'}\\
&\le& (L\vee1)\nrm{ (x,y)-(x',y')}_\Z,
\eeq
where $a\vee b:=\max\set{a,b}$.
We conclude that any $f\in\F_L$
is $(L\vee1) < L+1 $-Lipschitz
under $\nrm{\cdot}_\Z$.

Since we restricted the domain and range of $\H_L$, respectively,
to the unit balls $B_\X$ and $B_\Y$,
the domain of $\F_L$
becomes
$B_\Z:=B_\X\times B_\Y$
and its range is $[0,2]$.
Let us recall some basic facts about the $\ell_2$ covering
of the $k$-dimensional unit ball
\beq
\cN(t,B_\X,\nrm{\cdot}) \le (3/t)^k;
\eeq
an analogous bound holds for
$\cN(t,B_\Y,\nrm{\cdot})$.
Now if $\C_\X$ is a collection of balls, each of diameter at most $t$,
that covers $B_\X$ and $\C_\Y$ is a similar collection covering $B_\Y$,
then clearly the collection of sets
\beq
\C_\Z:=\set{ E=F\times G \subset\Z: F\in \C_\X, G\in\C_\Y}
\eeq
covers $B_\Z$. Moreover, each $E\in\C_\Z$ is a ball
of diameter at most $2t$ in $(\Z,\nrm{\cdot}_\Z)$.
It follows that
\beq
\cN(t,B_\Z,\nrm{\cdot}_Z) \le (6/t)^{2k}.
\eeq

Finally, we endow $F_L$ with the $\ell_\infty$ norm, and
use a Kolmogorov-Tihomirov type covering estimate
(see, e.g., \citet[Lemma 5.2]{gottlieb13adaptive}):
\beq
\log \cN(t,F_L,\nrm{\cdot}_\infty) \le
(96(L+1)/t)^{2k}\log(8/t).
\eeq

We can now use
\citet[Theorem 4.3]{gottlieb13adaptive}):
\begin{theorem}
  Let $\F_L$ be the collection of {\em $L$-Lipschitz}
  $[0, 2]$-valued functions defined on a metric space $(\Z, \norm{\cdot}_Z)$
  with diameter $1$ and doubling dimension $d$.
  Then $\hat{R}_n(F_L; \Z) = O\big (\frac{L}{n^{1/(d+1)}}\big ) $.
\end{theorem}

Putting $d = k+\ell$ yields our generalization bound:
\begin{equation}
  Q_n(k,\ell,L) = R(h(z)) - \hat R_n(h(z)) = O\Big(\frac{L}{n^{\frac{1}{d+1}}}
  \Big).
\end{equation}

\section{Additional experiments.}
For completeness we add here the comparison of the results from the experiment for $f(x) = \sin(x)$ for $ x\in [-2\pi,2\pi]$.

\begin{table}
%[h!]
  \begin{center}
    \caption{ERM of the smoothing process}
%    \label{tab:res}
    \begin{tabular}{l|c|c|c|r} %
     training points & 100 & 200 & 500 & 1000 \\ \hline
     \textbf{Algorithm} & \multicolumn{4}{c}{\textbf{Avg. loss}} \\
      \hline
      MWU & 5.5e-09 & 1.6639e-08 & 3.7683e-08 & 6.8339e-08\\
      IntPt & 5.3e-16 & 5913495.3623 & & \\
  
    \end{tabular}
  \end{center}
\end{table}

\begin{table}
%[h!]
  \begin{center}
    \caption{Cross validation running time over the smoothing process in seconds}
%    \label{tab:res}
    \begin{tabular}{l|c|c|c|r} %
     training points & 100 & 200 & 500 & 1000 \\ \hline
     \textbf{Algorithm} & \multicolumn{4}{c}{\textbf{Avg. loss}} \\
      \hline
      MWU & 5.1933 & 11.3827 & 66.5677 & 335.0659\\
      IntPt & 3508.4048 & 3936.6494 & & \\
  
    \end{tabular}
  \end{center}
\end{table}

\begin{table}
%[h!]
  \begin{center}
    \caption{Single$^*$ smoothing process running time in seconds}
%    \label{tab:res}
    \begin{tabular}{l|c|c|c|r} %
     training points & 100 & 200 & 500 & 1000 \\ \hline
     \textbf{Algorithm} & \multicolumn{4}{c}{\textbf{Avg. loss}} \\
      \hline
      MWU & 0.78109   & 1.8081 & 8.1855 & 48.5127 \\
      IntPt & 114.5577   & 778.1016  & & \\
  
    \end{tabular}
  \end{center}
\end{table}

\begin{table}
%[h!]
  \begin{center}
    \caption{Extension avg loss}
%    \label{tab:res}
    \begin{tabular}{l|c|c|c|r} %
     training points & 100 & 200 & 500 & 1000 \\ \hline
     \textbf{Algorithm} & \multicolumn{4}{c}{\textbf{Avg. loss}} \\
      \hline
      MWU & 5.5e-09 & 1.09e-08 & 3.3e-08 & 7.3e-08\\
      IntPt & 0.2877 & 0.83248 & & \\
  
    \end{tabular}
  \end{center}
\end{table}

\begin{table}
%[h!]
  \begin{center}
    \caption{Extension running time for a all test set in seconds}
    \label{tab:res}
    \begin{tabular}{l|c|c|c|r} %
     training points & 100 & 200 & 500 & 1000 \\ \hline
     \textbf{Algorithm} & \multicolumn{4}{c}{\textbf{Avg. loss}} \\
      \hline
      MWU & 0.0026   & 0.0038 & 0.0061  & 0.0095  \\
      IntPt & 4.3164  & 1.824  & & \\
  
    \end{tabular}
  \end{center}
\end{table}

\begin{figure}
%[h]
  \centering
  	\begin{subfigure}[b]{0.4\linewidth}
    	\includegraphics[width=\linewidth] {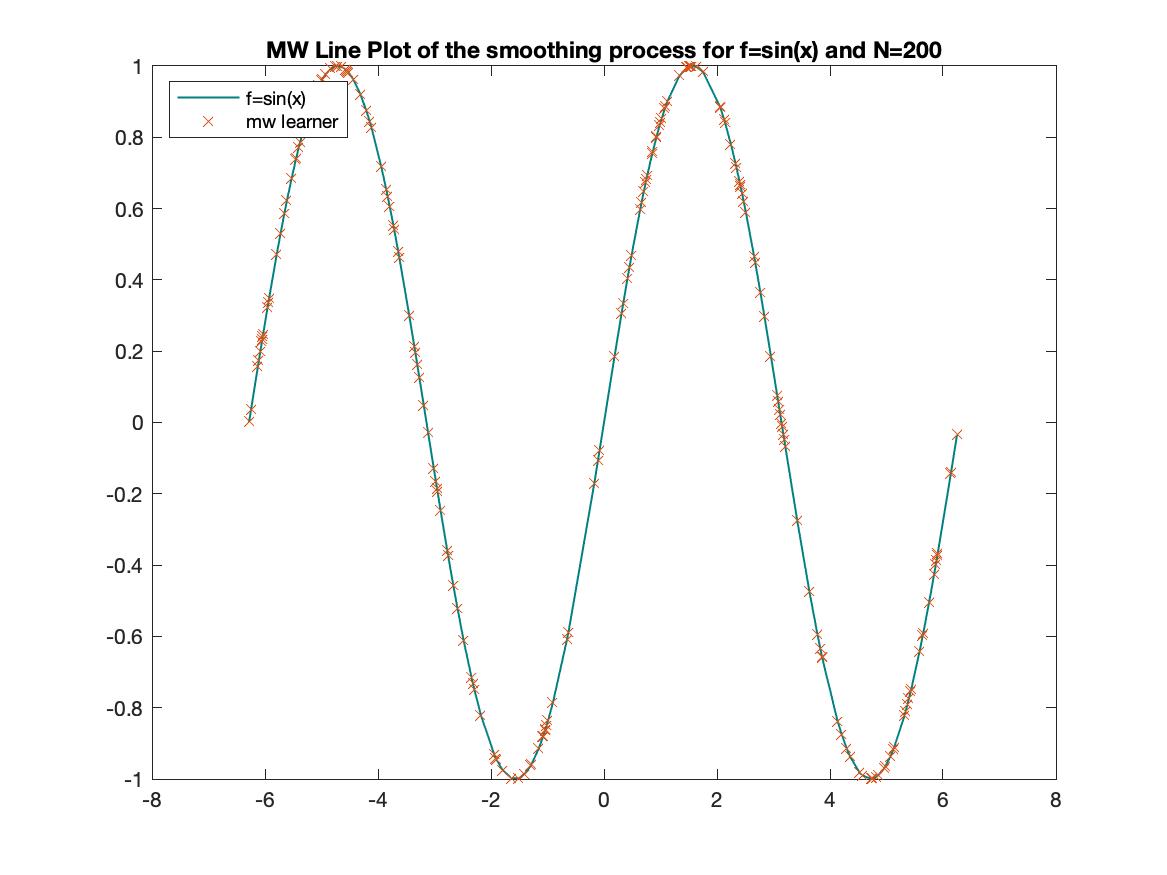}
    	\caption{}
  	\end{subfigure}
  	\begin{subfigure}[b]{0.4\linewidth}
    	\includegraphics[width=\linewidth] {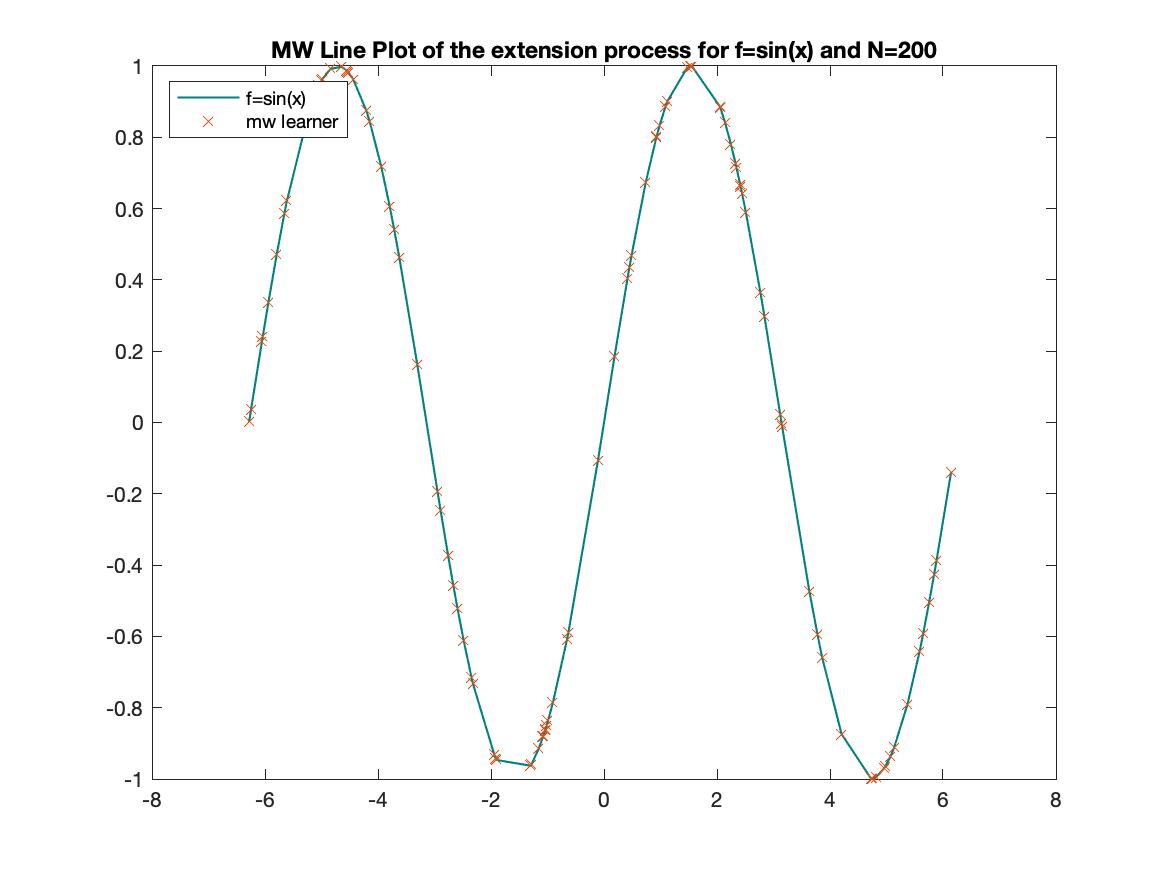}
		\caption{}
  	\end{subfigure}
  	\begin{subfigure}[b]{0.4\linewidth}
    	\includegraphics[width=\linewidth] {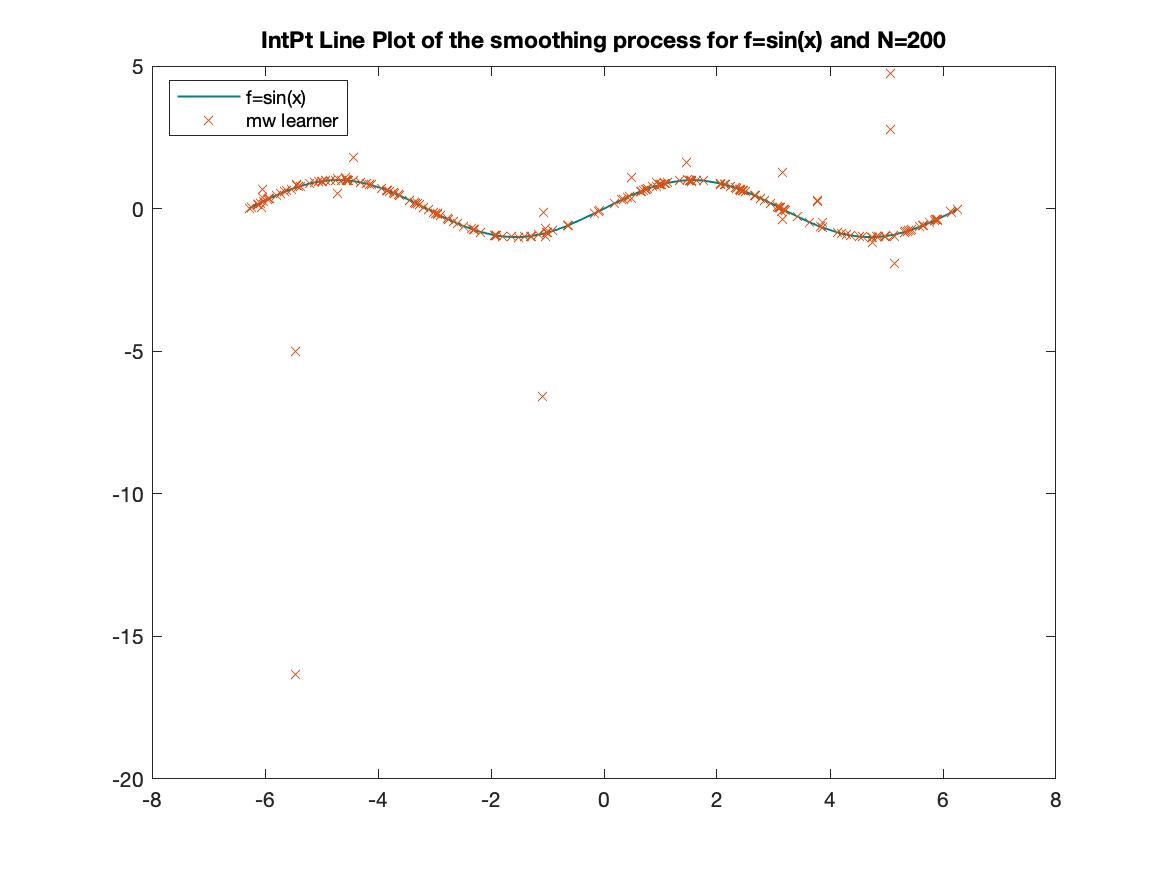}
		\caption{}
  	\end{subfigure}\begin{subfigure}[b]{0.4\linewidth}
    	\includegraphics[width=\linewidth] {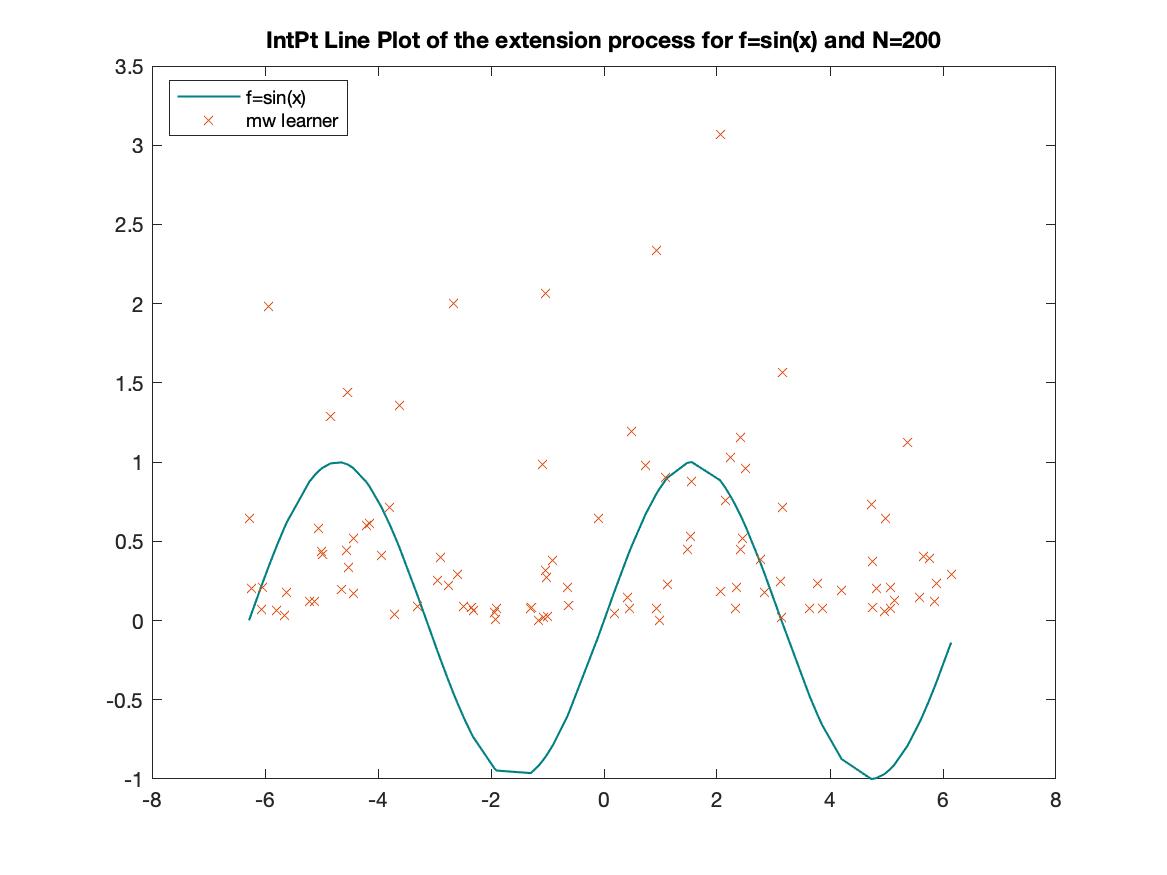}
		\caption{}
  	\end{subfigure}
  \caption{Figures (a)-(d) demonstrate the comparison, where training size is 200 random points, and $f=x^3$. Figures (a)-(b) are the smoothing and extension phases implemented with our MWU based algorithm while (c)-(d) are the results of Matlab's IntPt implementation. in All graphs the line represents the ground truth function while the X represent the estimation by the learned function}
  \label{fig:fig2}
\end{figure}

\end{document}